\documentclass[11pt]{article}
\usepackage{amsmath,amssymb,amsfonts} 
\usepackage{epsfig} \usepackage{latexsym,nicefrac,bbm}
\usepackage{xspace}
\usepackage{color,fancybox,graphicx,subfigure,fullpage}
\usepackage[top=1.25in, bottom=1.25in, left=1.25in, right=1.25in]{geometry}
\usepackage{tabularx} \usepackage{hyperref} 
\usepackage{pdfsync}
\usepackage[boxruled]{algorithm2e}
\usepackage{multicol}
\usepackage{authblk}

\newcommand{\eps}{\varepsilon}

\newcommand{\nhd}[1]{\partial #1}
\newcommand{\pr}{\mathbf{Pr}}

\newcommand{\st}{\mbox{\rm s.t. }}

\newcommand{\per}{\mathrm{Per}}
\newcommand{\kl}{\mathrm{KL}}
\newcommand{\cB}{\mathcal{B}}
\newcommand{\emax}{E_{\max}}

\newcommand\N{\mathbb N}
\newcommand\R{\mathbb R}
\newcommand\C{\mathbb C}
\newcommand{\inparen}[1]{\left(#1\right)}             
\newcommand{\insquare}[1]{\left[#1\right]}             
\newcommand{\inangle}[1]{\left\langle#1\right\rangle} 

\newcommand{\wt}[1]{\widetilde{#1}}

\newtheorem{theorem}{Theorem}[section]

\newtheorem{definition}{Definition}[section]
\newtheorem{lemma}[theorem]{Lemma}
\newtheorem{remark}[theorem]{Remark}

\newenvironment{proof}{\begin{trivlist} \item {\bf Proof:~~}}
   {\qed\end{trivlist}}

\def\FullBox{\hbox{\vrule width 6pt height 6pt depth 0pt}}

\def\qed{\ifmmode\qquad\FullBox\else{\unskip\nobreak\hfil
\penalty50\hskip1em\null\nobreak\hfil\FullBox
\parfillskip=0pt\finalhyphendemerits=0\endgraf}\fi}

\newenvironment{proofof}[1]{\begin{trivlist} \item {\bf Proof
#1:~~}}
  {\qed\end{trivlist}}

\begin{document}

\title{ Belief Propagation, Bethe Approximation and Polynomials}

\date{}

\author[1]{ Damian Straszak}
\author[2]{Nisheeth K. Vishnoi}
\affil[1,2]{\small \'{E}cole Polytechnique F\'{e}d\'{e}rale de Lausanne (EPFL), Switzerland}

\maketitle

\begin{abstract}
Factor graphs are important models for succinctly representing probability distributions in machine learning, coding theory, and statistical physics. 
Several computational problems, such as computing marginals and partition functions, arise naturally when working with factor graphs.
Belief propagation is a widely deployed iterative method for solving these problems.
However, despite its significant empirical success, not much is known about the correctness and efficiency of  belief propagation.

Bethe approximation is an optimization-based framework for approximating partition functions.
While it is known that the stationary points of the Bethe approximation coincide with the fixed points of belief propagation, in general, the relation between the Bethe approximation and the partition function is not well understood.
It has been observed that for a few classes of factor graphs, the Bethe approximation always gives a lower bound to the partition function, which distinguishes them from the general case, where neither a lower bound, nor an upper bound holds universally.
This has been rigorously proved for permanents \cite{Vontobel13,GS14} and for attractive graphical models \cite{Ruozzi12}.

Here we consider bipartite normal factor graphs and show that if the local constraints satisfy a certain analytic property,  the Bethe approximation is a lower bound to the partition function. 
We arrive at this result by viewing factor graphs through the lens of polynomials.
In this process, we reformulate the Bethe approximation as a polynomial optimization problem.
Our sufficient condition for the lower bound property to hold is inspired by recent developments in the theory of real stable polynomials.
We believe that this  way of viewing factor graphs and its connection to real stability  might lead to a better understanding of belief propagation and factor graphs in general.
\end{abstract}

\newpage

\section{Introduction}
Several important classes of probability distributions studied in statistical physics, coding theory, and machine learning can be succinctly represented as factor graphs \cite{MM09,WJ08}.
Informally, they provide a way to describe complex, multivariate functions by specifying variables and relations between them in a form of a hypergraph \cite{KFL01}. 
In this context, of interest are the inference problem of estimating marginal probabilities of certain variables and the problem of estimating the partition function of such a factor graph.
In computer vision one applies such inference primitives to learn about objects in a stage being captured by several cameras~\cite{FPC00}.
They are also essential components for decoding algorithms for Low-Density Parity Check codes \cite{Gallager62,Tanner81}. 
In statistical physics, these problems are equivalent to learning properties of typical configurations of a given mechanical system~\cite{MM09}.
Due to  the practical relevance and broad applicability of such inference primitives, over several decades numerous approximate and heuristic methods have been developed to compute these quantities. 
Among them, the most widely deployed is the belief propagation method \cite{Gallager62,Pearl88}, which is an iterative {\it message passing algorithm} (or equivalently a discrete-time dynamical system) for computing marginals and partition functions. 
It is known that belief propagation provides exact answers when the considered factor graph is a tree~\cite{Pearl88} and gives decent approximations on locally tree-like graphs \cite{DM10}.
However, a general theory explaining the great empirical success of the belief propagation method is lacking.

\smallskip
Another, seemingly unrelated approach, with its roots in physics, is the Bethe approximation \cite{Bethe35,KKW53,Georgii11}.
It is based on computing the optimal value (called the Bethe partition function) to a certain continuous optimization problem and using it as an estimate of the true partition function.
There is a fundamental connection known between belief propagation and Bethe approximation -- the fixed points of the former arrive exactly as the stationary points of the optimization problem underlying the latter~\cite{YFW05}.
This provides a good grasp on the belief propagation algorithm, that is otherwise hard to reason about.
By establishing bounds on the Bethe partition function, one can deduce facts about the behavior of the belief propagation algorithm and, importantly, learn to some extent, where will it converge to.

\smallskip
Even though for real-world examples of factor graphs the Bethe partition function seems to provide a decent estimate to the partition function, there are known examples for which the approximation is arbitrarily bad \cite{WJ08,WTJS14}.
This is not a surprise, as the inference problems related to factor graphs can encode {\bf NP}-hard problems and even {\bf \#P}-hard problems (such as counting independent sets in a graph) can be seen as computing certain partition functions.
Another difficulty, which rules out several proof techniques for dealing with such relaxations is the fact that the underlying optimization problem is not convex.
For this reason, it is hard to expect a characterization of factor graphs for which the Bethe approximation can be related to the true partition function.
Instead, there are efforts to describe viable sufficient conditions under which some relation can be established.
For factor graphs representing permanents, it has been proved that the Bethe approximation is a lower bound to the true partition function~\cite{Vontobel13,Gurvits11,GS14}.
A similar phenomenon has been observed and conjectured to hold for {\it log-supermodular} factor graphs~\cite{SudderthWW07} and a positive resolution was proposed by~\cite{Ruozzi12}.

\medskip
We propose a new, alternative view on factor graphs via the lens of polynomials.
Specifically, we introduce a natural way of representing local functions as polynomials, so that the Bethe approximation can be restated as a polynomial optimization problem.
This allows us to relate properties of the underlying polynomials to the behavior of the Bethe approximation.
We state a natural analytic condition under which the Bethe partition function lower-bounds the true partition function.
The condition is inspired by recent developments in the theory of {\it real stable polynomials}~\cite{BBL09,BB09,BB09b} and in particular by recent polynomial approaches to partition functions~\cite{AO17,SV17} (see Remark~\ref{rem:ao_sv} for a comparison) based on ideas from~\cite{Gurvits06}.
In its simplest form, it requires all the polynomials underlying the factor graph to be real stable.
Interestingly, such factor graphs are necessarily {\em repulsive} or {\it log-submodular}, which complements the lower bounds obtained by~\cite{Ruozzi12} -- for {\em attractive} or {\it log-supermodular} models.
We believe that this framework based on polynomials might be used to establish similar bounds for different classes of factor graphs and more generally to answer different questions about the Bethe approximation and the belief propagation algorithm.

\section{Factor Graphs and Bethe Approximation}\label{sec:defs}
\subsection{Factor Graphs}
We work with probability distributions represented by Normal Factor Graphs (NFGs).
In an NFG $G=(F,E,\{g_a\}_{a\in F})$, there is a set of factors (or nodes) $F$ and a set of variables (or edges) $E$. 
Every edge $e\in E$ connects exactly two factors. 
The set of edges incident to a factor $a\in F$ is denoted by $\nhd{a} \subseteq E$. 
The last component of $G$ is a collection of {\it local functions} $\{g_a\}_{a\in F}$. 
Every such function $g_a$ takes as input a binary string of length $|\nhd{a}|$ and outputs a non-negative number, in other words $g_a: \{0,1\}^{\nhd{a}} \to \R_{\geq 0}$. 
For a given vector $\sigma \in \{0,1\}^E$ and any set of edges $S\subseteq E$ we denote by $\sigma_{S}$ the sub-vector of $\sigma$ of length $|S|$ indexed by edges in $S$. 
Edges are to be thought of as variables that can take one of two possible values: $0$ or $1$. 
Then the set of all possible configurations of $G$ is $\{0,1\}^E$. 
Consider the probability distribution $p$ on $\{0,1\}^E$ by setting
\begin{equation}
\begin{aligned}
	p_{\sigma}&:= \frac{\prod_{a\in F} g_a({\sigma}_{\nhd{a}})}{Z(G)}~~~~~~~\mbox{for }{\sigma}\in \{0,1\}^E,\\  Z(G)&:= \sum_{{\sigma}\in \{0,1\}^E}\prod_{a\in F} g_a({\sigma}_{\nhd{a}}).
\end{aligned}
\label{eq:p}
\end{equation}
It is always assumed that $Z(G) \neq 0$, in which case $p$ is a well defined probability distribution over configurations. 
The focus here is on the problem of estimating $Z(G)$ for a given normal factor graph $G$. 

Note that in a related model of {\it factor graphs}, variables are represented by variable nodes, whereas in the model  considered here they are represented by edges. 
However, a simple reduction shows that these two models are equivalent~\cite{FV11}. 
We choose to work with normal factor graphs to allow a cleaner statement of results.

\subsection{Bethe Approximation}

The Bethe approximation is a popular heuristic called  for computing $Z(G)$. 
It is based on computing a quantity $Z_B(G)$ -- called the Bethe partition function of $G$ -- as a solution to a continuous optimization problem defined with respect to $G$. 
To derive the Bethe approximation, one begins with the following convex program
\begin{equation}
\begin{aligned}
	\sup_{q}~~ &\sum_{\sigma} q_\sigma \log \frac{g(\sigma)}{q_\sigma}\\
	\st ~~ &  \sum_{\sigma \in \{0,1\}^E} q_\sigma=1,&\\
	& q\geq 0. &
\end{aligned}
\label{eq:entropy_prog}	
\end{equation}
where $g(\sigma)= \prod_{a\in F} g_a(\sigma_{\nhd a})$. 
It is not hard to prove that the above program has an optimal solution $q^\star = p$ (with $p$ as in~\eqref{eq:p}), and the optimal value is $\log Z(G)$. 
Thus the problem of computing the partition function is reduced to solving the program~\eqref{eq:entropy_prog}. 
This reduction, however, does not seem to make the problem any easier, as the number of variables in~\eqref{eq:entropy_prog} is exponential. 
Thus, various heuristics have been proposed on how to reduce the number of variables in~\eqref{eq:entropy_prog} so as to make this approach of estimating $\log Z(G)$  feasible. 
The Bethe approximation has variables $\beta_e \in [0,1]$ for $e\in E$, which are the marginals of the distribution $\{q_\sigma\}_{\sigma \in \{0,1\}^E}$, more formally we think of $\beta_e$ as $\pr[X_e=1]$ where $X \in \{0,1\}^E$ is distributed according to $q$. 
Similarly one introduces variables representing marginals over factors, i.e. for $a\in F$ we have a vector $\alpha_a$ which is a probability distribution over local configurations $\{0,1\}^{\nhd a}$, and its interpretation is that $\alpha_a(c) = \pr[X_a=c]$. 
To simplify the program~\eqref{eq:entropy_prog} the following assumption is made about the form of the distribution $\{q_\sigma\}_{\sigma \in \{0,1\}^E}$ 
\begin{equation}\label{eq:approx_q}
\forall_{\sigma \in \{0,1\}^E} ~~~~~~q_\sigma = \frac{\prod_{a\in F} \alpha_a(\sigma_{\nhd a})}{\prod_{e\in E} \beta_e^{\sigma_e}(1-\beta_e)^{1-\sigma_e}}.
\end{equation}
The intuition behind such a form of $q_\sigma$ is that one might (for simplicity) assume independence between factors and calculate the probability of a global configuration as a product of probabilities over local configurations of factors. 
The term in the denominator can be thought of as a correction term, as every edge is ``taken twice into account'' in the numerator. 
Another way of motivating~\eqref{eq:approx_q} is to observe that when the graph $G$ is a tree, then the probability function can be written in this form and, wishfully, one may expect that for other graphs it might serve as a good estimate.
Assuming such a special form of $q$, the program~\eqref{eq:entropy_prog} reduces to
\begin{equation}
\begin{aligned}
	\sup_{\alpha, \beta}~~ &\sum_{a\in F} \sum_{c\in \{0,1\}^{\nhd a}} \alpha_a(c) \log \frac{g_a(c)}{\alpha_a(c)} -\sum_{e\in E}H(\beta_e)\\
	\st ~~ &  (\alpha, \beta) \in \Gamma(G)&\\
\end{aligned}
\label{eq:bethe}
\end{equation}
where $H$ is the binary entropy function (i.e., $H(x) = -x\log x - (1-x) \log (1-x)$ for $x\in [0,1]$) and $\Gamma(G)$ is the set of all marginal vectors which satisfy {\it local agreement constraints} (it is thus called the {\it pseudo-marginal polytope}). 
This means that $\beta_e$ and $\alpha_a$ are as above and they satisfy:
$$\sum_{c\in \{0,1\}^{\nhd a}} \alpha_a(c) \cdot c = \beta_a~~~~~~~\mbox{for every }a\in F.$$
The optimal value of~\eqref{eq:bethe} is called the Bethe partition function and its exponential is denoted by $Z_B(G)$.
One expects that $Z_B(G)$ is a decent approximation to $Z(G)$, which has been confirmed empirically for various examples of factor graphs. 

However, in general, $Z_B(G)$ can be an arbitrarily bad approximation to $Z(G)$, as for instance it might be positive for some cases where $Z(G)=0$. 
From a theoretical viewpoint, not much is known about the behavior of Bethe approximation. 
The main source of difficulty in understanding this relaxation is its non-convexity, which in particular manifests itself in multiple local optima. 
In the current paper we derive some sufficient conditions under which the Bethe partition function lower-bounds the true partition function. 

\subsection{Related Work} 
The notion of free energy that appears as the objective in the Bethe partition function was formulated in~\cite{Bethe35} in the physics literature. 
See also~\cite{Mori13} and references therein for more historical notes on Bethe approximation.
The correspondence between Bethe approximation and the belief propagation algorithm was explicitly derived in~\cite{YFW05}. 
This combined with the work~\cite{Pearl88} on the belief propagation method implies that Bethe approximation gives exact values of the partition function on tree factor graphs. 
It is also known that Bethe partition function gives precise estimates in the asymptotic sense on locally tree-like graphs \cite{DM10}.

In the work~\cite{CC06}, the loop series expansion of the Bethe partition function was introduced, which is a tool to study the relation between the Bethe partition function and the true partition function.
In~\cite{CC06} the loop expansion was used to prove that Bethe approximation gives a good estimate on the number of independent sets on graphs with small maximum degree and large girth.

The problem of computing permanents of nonnegative matrices  has been also intensively studied in the context of  Bethe approximation~\cite{WC10, Vontobel13, Gurvits11, GS14}.
Recall that the permanent of a matrix $A\in \R^{n\times n}$ is defined to be $$\per(A):=\sum_{\sigma \in S_n} \prod_{i=1}^n A_{i,\sigma(i)}$$ and the problem of computing it is a canonical example of a {\bf \#P}-hard problem~\cite{Valiant79}, hence no polynomial time exact algorithm is expected to exist. 
This problem can be formulated in a natural way as evaluating a certain partition function $Z(G)$~\cite{WC10, Vontobel13} and hence one can investigate the question on how well the Bethe partition function does approximate permanents. 

It has been observed~\cite{Vontobel13} that unlike in the general case, for permanents the program~\eqref{eq:bethe} is convex. 
This allows one to analyze the optimality via KKT conditions and to conclude that $Z_B(G) \leq Z(G)$ using a  permanental inequality due to~\cite{Schrijver98}. 
The success of this approach crucially relies  on the existence of a convex form of the Bethe approximation, this seems to be an exception rather than a rule among various factor graphs.

The Bethe approximation was also studied in the context of the {\it Ising model}~\cite{SudderthWW07}, and shown to lower-bound the true partition function for the {\it ferromagnetic} case under certain technical assumption.
This result was extended by~\cite{Ruozzi12} to the class of all {\it log-supermodular} (also called {\it attractive}) factor graphs.
A factor graph is called log-supermodular if every local function is log-supermodular, i.e., for every $a\in F$ we have
$$\forall_{\sigma, \tau \in \{0,1\}^{ \nhd a}}~~~~g_a(\sigma)\cdot g_a(\tau) \leq g_a(\sigma \vee \tau) \cdot g_a(\sigma \wedge \tau),$$
where $\vee$ and $\wedge$ denote entry-wise OR and entry-wise AND respectively. 
The proof is based on the following combinatorial characterization of the Bethe approximation, due to~\cite{Vontobel13,Vontobel13a}. 
It says that
$$Z_B(G) = \limsup_{k \to \infty} \sqrt[k]{\mathbb{E}_{H \in G^{(k)}} Z(H)},$$
where $G^{(k)}$ is the set of $k$-covers of the factor graph $G$, and the expectation is over a uniformly random choice of $H$ in $G^{(k)}$ (for details we refer to~\cite{Vontobel13a}). 
It follows that in order to prove that $Z_B(G) \leq Z(G)$ for a given factor graph $G$, it is enough to prove that for every $k\in \N$
\begin{equation}\label{eq:cover_ineq}
Z(H) \leq Z(G)^k,~~\mbox{for every $k$-cover $H$ of $G$}.
\end{equation} 
This is the main idea behind the reasoning of~\cite{Ruozzi12}; the inequality~\eqref{eq:cover_ineq} is then proved using a certain generalization of the four function theorem~\cite{AD78}. 
In the context of attractive models, several conjectures regarding similar lower bounds were stated in~\cite{Watanabe11}, out of which only one (for independent sets on bipartite graphs) has been so far resolved (by the above result of~\cite{Ruozzi12}).

Finally we mention that this paper is inspired by recent developments in the theory of real stable polynomials~\cite{BBL09,BB09,BB09b} and the works of~\cite{Gurvits06,SV17,AO17}, where several polynomial based relaxations are considered; for details, we refer the reader to Remark~\ref{rem:ao_sv}.

\section{Our Contribution}
\subsection{Polynomial Form of Bethe Approximation}
The main conceptual result of this paper is  a new approach to prove inequalities between the Bethe partition function and the true partition function. We start by presenting an alternative view on the Bethe approximation -- through the lens of polynomials. 
Towards this, let us first define the polynomial representation of local functions. 
For any $a\in F$ we define a multivariate polynomial $h_a$ over a set of $|\nhd{a}|$ variables $x_a:=\{x_{a,e}\}_{e\in  \nhd a}$ as follows
\begin{equation}\label{eq:poly_f}
h_a(x_a) := \sum_{{\sigma}\in \{0,1\}^{\nhd{a}}} h_{a,{\sigma}} x_a^{{\sigma}},
\end{equation}
where $x_a^{\sigma}$ is a monomial defined as $x_a^{\sigma}:=\prod_{e\in \nhd{a}} x_{a,e}^{{\sigma}_e}$ and the coefficient $h_{a,{\sigma}}$ is given by $h_{a,{\sigma}} := g_a(\sigma)$. 
We prove the following, alternative characterization of the Bethe partition function as a polynomial optimization problem. 
In the statement below we use the convenient notation that for two vectors $x,{\sigma}\in \R^k$, $x^\sigma := \prod_{i=1}^k x_i^{\sigma_i}$.
\begin{theorem}[\textbf{Bethe Approximation via Polynomials}]\label{theorem:poly}
Let $G$ be a normal factor graph with a set of factors $F$ and a set of variables $E$. For every factor $a\in F$ let $h_a$ be the corresponding $|\nhd{a}|$-variate polynomial. Then the Bethe partition function can be written as
$$Z_B(G) = \max_{\beta \in [0,1]^E} \insquare{ \prod_{e\in E} \beta_e^{\beta_e} (1-\beta_e)^{1-\beta_e} \inf_{x>0} \prod_{a\in F} \frac{h_a(x_a)}{x_a^{\beta_a}} }$$
\end{theorem}
In the above statements $x$ stands for a vector which collects all variables $x_{a,e}$ for $a\in F$ and $e\in \nhd{a}$.
The proof of Theorem~\ref{theorem:poly} appears in Section~\ref{sec:bethe_poly}. 
It is established by adapting a dual view on the max-entropy program which defines the Bethe partition function. 

\subsection{Lower Bound on the Partition Function}

Technically, we prove that assuming a certain geometric condition on the factor graph $G$, the Bethe approximation provides a lower bound on the true partition function. 
This condition captures permanents as a special case (see the example provided in Section~\ref{sec:permanent}).
Below we state a simplified variant of the main technical result in terms of {\it local polynomials} $h_a$.
For a more general statement, which is expressed in the language of probability, as well as a proof of the below theorem, we refer to Section~\ref{sec:proofs}.

\begin{theorem}[\textbf{Lower Bound via Real Stability}]\label{thm:rs}
Let $G$ be a bipartite normal factor graph with a set of factors $F$ and a set of variables $E$. Assume that all the polynomials $h_a$ corresponding to local functions $g_a$ (for $a\in F$) are real stable. Then it holds that
$Z_B(G) \leq Z(G).$
\end{theorem}
\noindent
A few comments are in order. 
In the statement above we assume that the NFG $G$ is bipartite. 
This might seem to be restrictive, but as it turns out, every NFG can be converted into an equivalent bipartite form, with at most a double growth in size, hence no real restriction is put on $G$ with this assumption. 
The key condition we require is {\it real stability} of the underlying polynomials.

Real stability is a geometric condition on the location of zeros of a polynomial, which generalizes real-rootedness. 
We say that a polynomial $h\in \R[x_1, \ldots, x_m]$ is real stable if  none of its roots $z=(z_1, \ldots, z_m) \in \mathbb{C}^m$ satisfies: $\Im(z_i)>0$ for every $i=1,2,\ldots, m$. 
Real stable polynomials have recently found numerous applications in mathematics~\cite{BB09, MSS13b} and computer science~\cite{Gurvits06, MSS13a,AO15,NS16,SV17, AO17}  (see also surveys~\cite{Wagner11,Pemantle11,Vishnoi-Survey}).

We remark  that coefficients of multi-affine real stable polynomials are known to be given by log-submodular set functions (see~\cite{Wagner11}), which corresponds to the following assumption on local functions $g_a$ for $a\in F$
$$\forall_{\sigma, \tau \in \{0,1\}^{ \nhd a}}~~~~g_a(\sigma)\cdot g_a(\tau) \geq g_a(\sigma \vee \tau) \cdot g_a(\sigma \wedge \tau).$$
This demonstrates that Theorem~\ref{thm:rs} addresses the opposite case when compared to the result of~\cite{Ruozzi12}, where an analogous result for log-supermodular functions is proved.
These two assumptions turn out to imply significantly different properties of the underlying factor graphs. 

One interesting aspect that is worth mentioning here is that, under log-supermodularity, {\it feasible fractional configurations} are easy to round to {\it integral configurations}.
More precisely, given a point $(\alpha, \beta) \in \Gamma(G)$ whose objective value in the Bethe approximation is finite (larger than $-\infty$), one can obtain (by just rounding up all entries of $\beta$) a configuration $\sigma \in \{0,1\}^E$ such that $g(\sigma)>0$. 
Such a procedure might fail in finding a feasible configuration when G is log-submodular (i.e., the resulting $\sigma$ has $g(\sigma)=0$).
In fact, finding a feasible configuration in such models (even assuming real stability of local polynomials) might be a nontrivial task, even {\bf NP}-complete if no assumptions on the local functions are made.
It turns out in particular, that for the case of permanents, the Bethe approximation is implicitly solving a nontrivial combinatorial optimization problem of detecting if a bipartite graph has a perfect matching. 

\begin{remark}[Upper Bound]
Using the characterization from Theorem~\ref{theorem:poly} one can prove that $Z(G) \leq 2^m \cdot Z_B(G)$.
Indeed, by plugging in $\beta:= \sigma \in \{0,1\}^E$ the term $\prod_{e\in E} \beta_e^{\beta_e} (1-\beta_e)^{1-\beta_e}$ is equal to $1$ and we obtain $$Z_B(G)\geq \sum_{\tau \leq \sigma} g(\tau) \geq g(\sigma)$$ (here by $\tau \leq \sigma$ we mean an entry-wise inequality).  
Hence, altogether, under the assumptions of Theorem~\ref{thm:rs} the Bethe partition function provides a $2^m-$approximation to the true partition function.
\end{remark}

\subsection{Discussion} In this paper we propose a new approach for establishing bounds on the partition function for graphical models based on polynomial techniques. 
This work is inspired by recent developments in the theory of stable polynomials~\cite{Gurvits06,BB09,AO17, SV17} and is an attempt to expand the scope of applicability of these tools. 
While our result seems to require real stability (with respect to the upper-half complex plane) of the underlying polynomials to deduce the desired bound, we believe that other forms of stability, such as stability with respect to a disc, or other analytic assumptions on the polynomials might yield other nontrivial bounds.

Finally, we note  that real stability also improves the computational properties of the Bethe approximation. 
Indeed, the fact that the function $x\mapsto \log p(x)$ is concave, for a real stable polynomial $p\in \R_{\geq 0}[x_1, \ldots, x_m]$,  can be used to show efficient computability of certain relaxations, similar to the Bethe partition function in the polynomial form (see~\cite{SV17, AO17}).
This might eventually lead to designing relaxations which match or even outperform Bethe approximation, while having provably correct and efficient algorithms.

\section{Bethe Approximation via Polynomials}\label{sec:bethe_poly}
In this section we derive an equivalent form of the Bethe partition function -- stated in terms of a polynomial optimization problem.
\subsection{Local Functions as Polynomials}
Consider a NFG $G=(F,E,\{g_a\}_{a\in F})$. In this paper we view the local functions $g_a$  (for $a\in F$) as polynomials. More formally, given a function $g_a: \{0,1\}^{\nhd{a}}\to \R_{\geq 0}$, we define the corresponding polynomial representation of $g_a$ as an $|\nhd a|$-variate polynomial $h_a(x_a)$ over variables $\{x_{a,e}\}_{e\in \nhd{a}}$ given by the formula
$$h_a(x_a) = \sum_{{\sigma}\in \{0,1\}^{\nhd{a}}} h_{a,{\sigma}} x_a^{{\sigma}},$$
where $x_a^{\sigma}$ denotes $\prod_{e\in S} x_{a,e}^{{\sigma}_e}$ and $h_{a,{\sigma}}=g_a({\sigma})$ is the value of the function $g_a$ at ${\sigma}$. Note that even if two factors $a,b\in F$ share an edge $e\in E$, the variables of $h_a$ and $h_b$ are still pairwise different.

\subsection{Bethe Approximation as Polynomial Optimization}

Let $G=(F,E,\{g_a\}_{a\in F})$ be a NFG. Denote by $H(\beta)$ the negative entropy of $\beta\in [0,1]^E$, i.e., $$H(\beta) :=- \inparen{\sum_{e\in E} \beta_e \log \beta_e+(1-\beta_e) \log (1-\beta_e)}.$$ We use $\kl(p,q)$ to denote the KL-divergence between two nonnegative vectors $p,q\in \R^{k}_{\geq 0}$ (typically probability distributions), 
$$\kl(p,q):= \sum_{i=1}^k p_i \log \frac{p_i}{q_i}.$$
 The Bethe approximation problem can be then rewritten as
$$\log Z_B(G) = \max_{(\alpha, \beta)\in \Gamma(G)} -\sum_{a\in F} \kl(\alpha_a, g_a) -H(\beta),$$
where $\Gamma(G)$ is the pseudo-marginal polytope, as introduced in Section~\ref{sec:defs}. 
We define the following entropy maximization problem.
\begin{definition}
Let $f:\{0,1\}^k \to \R_{\geq 0}$ be any function with $C(f)=\{{\sigma}\in \{0,1\}^k: f({\sigma})>0\}$ and $\beta\in [0,1]^k$ be any vector.  We define $\emax(f,\beta)$  to be the optimal value of the following optimization problem over vectors $\alpha \in \R^{C(f)}$
\begin{equation}
\begin{aligned}
	\max_{\alpha}~~ &-\kl(\alpha, f)\\
	\st ~~ &  \sum_{c\in C(f)} \alpha(c)\cdot c=\beta,&\\
	& \sum_{{\sigma}\in C(f)} \alpha_{\sigma} = 1\\
	& \alpha\geq 0. &
\end{aligned}
\label{eq:kl}	
\end{equation}
In case when no $\alpha$ satisfies the above constraints, we set $\emax(f,\beta)=-\infty$.
\end{definition}

\begin{lemma}\label{lemma:reform1}
For every normal factor graph $G$, the Bethe approximation can be stated equivalently as
$$\log Z_B(G) = \max_{\beta \in [0,1]^E} \sum_{a\in F}\emax(g_a,\beta_a)-H(\beta).$$
\end{lemma}
\begin{proof}
The objective of the Bethe approximation $-\sum_{a\in F} \kl(\alpha_a, g_a) -H(\beta)$ has separated $\alpha$ and $\beta$ variables, however they are implicitly coupled because of the $(\alpha, \beta)\in \Gamma$ constraint. For a fixed $\beta$ and a factor $a\in F$ the constraint on $\alpha_a$ following from $(\alpha, \beta) \in \Gamma$ is
$$\sum_{c\in C_a, c_e=1} \alpha_a(c) = \beta_e~~~~~~~~\mbox{for every }e\in E.$$ 
This can be equivalently written in the vector form as
$$\sum_{c\in C_a} \alpha_a(c) \cdot c = \beta_a.$$
Note that maximizing $-\kl(\alpha_a, g_a)$ under this constraint gives us exactly $\emax(g_a, \beta_a)$.
\end{proof}
\noindent 
The lemma below explains how does the entropy maximization problem underlying $\emax$ relate to polynomial optimization.
\begin{lemma}\label{lemma:entro_poly}
Let $f:\{0,1\}^k \to \R_{\geq 0}$ and  $\beta\in [0,1]^k$ be any vector. Define a $k$-variate, multi-linear polynomial $h\in \R[x_1, \ldots, x_k]$ to be $h(x) = \sum_{{\sigma}\in \{0,1\}^k } h_{{\sigma}} x^{\sigma}$ with $h_{\sigma}:= f({\sigma})$.  We have
\begin{equation}
	\emax(f,\beta)=\inf_{x\in \R^k, x>0}~~\log h(x) - \sum_{i=1}^k \beta_i \log x_i.
\end{equation}
\end{lemma}
\begin{proof}
A proof follows by applying strong duality to the max-entropy program~\eqref{eq:kl}. For details, see \cite{SinghV14, SV17}.
\end{proof}
\noindent 
Theorem~\ref{theorem:poly} is now a simple consequence of the above established results.
\begin{proofof}{of Theorem~\ref{theorem:poly}}
From Lemma~\ref{lemma:reform1} we have 
$$\log Z_B(G) = \max_{\beta \in [0,1]^E} \sum_{a\in F}\emax(g_a,\beta_a)-H(\beta).$$
Next, by Lemma~\ref{lemma:entro_poly} this can be rewritten as
$$\log Z_B(G) = \max_{\beta\in [0,1]^E } \sum_{a\in F}\inf_{x_a>0}\inparen{\log h_a(x_a)- \sum_{e\in \partial a} \beta_e \log x_{a,e} }-H(\beta).$$
By taking exponentials on both sides
$$Z_B(G) = \max_{\beta\in [0,1]^E} \prod_{e\in E} \beta_e^{\beta_e} (1-\beta_e)^{1-\beta_e} \prod_{a\in F} \inf_{x_a>0} \frac{h_a(x_a)}{\prod_{e\in \partial a} x_{a}^{\beta_a}}.$$
\end{proofof}

\section{Proof of the Lower Bound}\label{sec:proofs}

To prove Theorem~\ref{thm:rs} we first formulate a more general condition which we call IPC, and prove that under IPC, the inequality $Z_B(G) \leq Z(G)$ holds. Afterwards we conclude the proof by showing that the assumption of Theorem~\ref{thm:rs} implies that IPC is satisfied.
\subsection{The IPC}

To state IPC we need to introduce some notation related to the bipartite structure of the factor graph $G=(F,E)$.
Let the set of factors $F$  be partitioned into two sets $L$ and $R$ such that no edges go between factors within $L$ or within $R$, only between these two sets.
Next, for any ${\sigma}\in \{0,1\}^E$ we define
$$l_{\sigma} = \prod_{a\in L} g_a({\sigma}_{\nhd{a}}) ~~~~~~~~\mbox{and} ~~~~~~~~~r_{\sigma} = \prod_{a\in R} g_a({\sigma}_{\nhd{a}}).$$
Furthermore we define the normalized variants of $l$ and $r$ to be $p^L_{\sigma} = \frac{l_{\sigma}}{\sum_{{\sigma}'}l_{{\sigma}'}}$, $p^R_{\sigma} = \frac{r_{\sigma}}{\sum_{{\sigma}'} r_{{\sigma}'}}$.
We refer to $p^L, p^R$ as to the distributions induced by $L$ (the ``left'' side of the bipartition) and induced by $R$ (the ``right'' side of the bipartition) respectively.

We are now ready to state a condition on the pair of distributions $(p^L, p^R)$ which will turn out sufficient for the inequality $Z_B(G) \leq Z(G)$ to hold. 
\begin{definition}[Iterated Positive Correlation]\label{def:propx}
Let $q,r$ be probability distributions over $\{0,1\}^m$ and let $X,Y \in \{0,1\}^m$ be distributed according to $q$ and $r$ respectively. Define the event $EQ_k$ to be $X_j=Y_j$ for all $j=1,2, \ldots, k$. For any two sequences of positive reals $s\in \R_{>0}^m$ and $t\in \R_{>0}^m$ and for any pair  $A,B\in \{0,1\}$ define
$$E_k(A,B) = \mathbb{E} \insquare{\prod_{j=k+1}^m s_j^{X_j} t_j^{Y_j} \cdot \mathbbm{1}_{X_k=A}\cdot \mathbbm{1}_{Y_k=B} \bigg|  EQ_{k-1}}$$
Where the expectation is over $X$ and $Y$, assuming $X,Y$ are independent. 
We say that the pair of distributions $(q,r)$ satisfies the Iterated Positive Correlation (IPC) property if
$$E_k(0,1) \cdot E_k(1,0)  \leq E_k(0,0)  \cdot E_k(1,1) $$
for every $k\in [m]$ and  for every $s,t \in \R^{m}_{>0}$.
\end{definition}
Note that in the definition above we implicitly assume that $\pr[EQ_m]\neq 0$, as otherwise some conditional expectations above might not be well defined.
For the setting which we have in mind, this corresponds to the assumption that $Z(G) \neq 0$.

To gain some intuition about the IPC property it is instructive to examine the special case when $s_1=\ldots = s_m = t_1 = \ldots = t_m=1$. Under the notation $p_k(A,B) := \pr[X_k=A\wedge Y_k=B| EQ_{k-1}]$  we obtain
$$p_k(0,1) \cdot p_k(1,0) \leq p_k(0,0)\cdot p_k(1,1),$$
which can be seen as a form of iterated (as $k=1,2,\ldots, m$) positive correlation between subsequent $X_k$'s and $Y_k$'s. In other words, it quantifies, in a certain sense the fact that conditioned on $X_i=Y_i$ for $i=1,2, \ldots, k-1$, it is more likely to see $X_k=Y_k$ rather than $X_k \neq Y_k$. 
We are now ready to state the main technical lemma of the paper, which asserts that if a NFG $G$ satisfies IPC then $Z_B(G) \leq Z(G)$. 

\begin{lemma}\label{lemma:main}
Let $G$ be a bipartite normal factor graph with a set of factors $F$, a set of variables $E$ and bipartition $F=L \cup R$. Let $p^L$ and $p^R$ be the distributions over $\{0,1\}^E$ induced by the left side and the right side of the bipartition of $G$ respectively. If the pair $(p_L, p_R)$ satisfies the IPC property then
$$Z_B(G) \leq Z(G) .$$
\end{lemma}

\noindent 
A proof of Lemma~\ref{lemma:main} appears in Section~\ref{ssec:proof_lb}.
To conclude Theorem~\ref{thm:rs} from the above it suffices to argue that the real stability assumption on local polynomials implies IPC. This is the subject of the next lemma

\begin{lemma}[Real Stability implies IPC]\label{lemma:rs-x}
Let $G$ be a bipartite normal factor graph with a set of factors $F$ and a set of variables $E$. 
Assume that all the polynomials $h_a$ corresponding to local functions $g_a$ (for $a\in F$) are real stable. Let $p^L$ and $p^R$ be the distributions over $\{0,1\}^E$ induced by the left side and the right side of the bipartition of $G$ respectively. 
Then the pair $(p^L, p^R)$ satisfies the IPC property.
\end{lemma}
\noindent 
The proof of Lemma~\ref{lemma:rs-x} appears in Section~\ref{ssec:proof_rs}. We are now ready to deduce Theorem~\ref{thm:rs}.
\begin{proofof}{of Theorem~\ref{thm:rs}}
Lemma~\ref{lemma:main} asserts that the inequality $Z_B(G) \leq Z(G)$ holds under IPC. Further, by Lemma~\ref{lemma:rs-x} the assumption of Theorem~\ref{thm:rs} (saying that local polynomials are real stable) implies that IPC holds. Thus the Theorem~\ref{thm:rs} follows.
\end{proofof}
\begin{remark}
We note that the IPC condition is significantly more general than the real stability assumption in~\ref{thm:rs} and there are examples of factor graphs which do not satisfy real stability, but IPC holds for them.
The downside of IPC might be however that there does not seem to be a simple way to verify it, especially since it is a global condition on the factor graph. On the other hand, the real stability assumption is only local and can be checked easily whenever the degrees of all factors are reasonably small.
\end{remark}
\subsection{Proof of the Lower Bound under IPC}\label{ssec:proof_lb}

In this section, the following linear operator on the set of polynomials is used.

\begin{definition}
Let $h(z_0, z, y_0, y)$ be a real polynomial with $z_0, y_0$ being single variables and $y,z$ being tuples of variables. Define 
$$\Phi_{z_0, y_0} (h) := (1+\partial_{z_0} \partial_{y_0}) h \mid_{z_0=y_0=0}.$$
\end{definition}
In other words, $\Phi_{z_0, y_0}$ first applies the differential operator $(1+\partial_{z_0} \partial_{y_0})$ to $h$ and then sets $z_0=y_0=0$; the result is a polynomial in the variables $(y,z)$. 

The lemma below explains how the IPC property is related to polynomials.
\begin{lemma}\label{lemma:poly-and-propx}
Let $q, r$ be distributions over $\{0,1\}^m$. Define the polynomials $q(z) := \sum_{{\sigma}\in \{0,1\}^m} q_{\sigma} z^{\sigma}$ and $r(y) := \sum_{\tau \in \{0,1\}^m} r_\tau  y^\tau $. Further, for every $k=0,1, \ldots, m$ let  
$$
f_k(z_{k+1}, \ldots, z_m, y_{k+1}, \ldots, y_m) := 
\Phi_{z_{k}, y_{k}} \cdots \Phi_{z_2, y_2} \Phi_{z_1, y_1} \insquare{q(z) \cdot r(y)}
$$
For any number $k=1,2, \ldots, m+1$ and for any two sequences $a,b\in \R_{\geq 0}^{m-k}$ of  non-negative numbers, the polynomial $f_{k-1}(z_k, a_{k+1}, \ldots, a_{m}, y_k, b_{k+1}, \ldots, b_m)$ is of the form
$$h(z_k, y_k) = h_{00}+h_{10}z_k + h_{01} y_k + h_{11}z_ky_k,$$
where (up to scaling) $h_{cd} = E_k(c,d)$ for every $c,d\in \{0,1\}$ (as in Definition~\ref{def:propx} with $a=s$ and $b=t$). 
\end{lemma}

\begin{proof}
We start by providing explicit formulas for the coefficients of $f_{k-1}$. Note first that all the operators $\Phi_{z_i, y_i}$ are linear. Hence it is enough to consider only one monomial $\prod_{i=1}^m z_i^{{\sigma}_i} \prod_{i=1}^m y_i^{\tau _i}$, for ${\sigma},\tau \in \{0,1\}^m$.
$$
\Phi_{z_{k-1}, y_{k-1}} \cdots \Phi_{z_2, y_2} \Phi_{z_1, y_1} \inparen{\prod_{i=1}^m z_i^{{\sigma}_i} \prod_{i=1}^m y_i^{\tau _i}} = 
\begin{cases}
\prod_{i=k+1}^m z_i^{{\sigma}_i} \prod_{i=k+1}^m y_i^{\tau _i} &\mbox{if }{\sigma}_1=\tau _1,\ldots  {\sigma}_k=\tau _k,\\
0 & \mbox{otherwise.}
\end{cases}$$
For this reason, the coefficient of  $\prod_{i=k}^m z_i^{{\sigma}_i} \prod_{i=k}^m y_i^{\tau _i}$ in $f_{k-1}$ is equal to
$$\sum_{u\in \{0,1\}^{k-1}} q_{u\widetilde{{\sigma}}} \cdot r_{u\widetilde{\tau }}.$$
where $\widetilde{{\sigma}}=({\sigma}_{k}, {\sigma}_{k+1}, \ldots, {\sigma}_m)$ and $\widetilde{\tau } = (\tau _{k}, \tau _{k+1}, \ldots, \tau _m)$. In the language of probability this coefficient is equal to the probability that 
\begin{align*}
X_{i}&=Y_{i},~~~~\mbox{ for } i=1,2, \ldots, k-1,\\
X_i &= {\sigma}_i ,~~~~\mbox{ for } i=k, k+1, \ldots, m,\\
Y_i &= \tau _i , ~~~~\mbox{ for } i=k,k+1, \ldots, m.
\end{align*}  when $X$ and $Y$ are distributed according to $q$ and $r$ respectively. Thus, when we consider $h_k(z_k,y_k)=f_{k-1}(z_k, a, y_k, b)$ for some $a,b \in \R_{\geq 0}^{m-k}$, the corresponding coefficients $h_{cd}$ are given by sums of the form
$$\sum_{u\in \{0,1\}^{k-1}} \sum_{\widetilde{{\sigma}} \in \{0,1\}^{m-k}} \sum_{\widetilde{\tau } \in \{0,1\}^{m-k}} q_{uc\widetilde{{\sigma}}}\cdot r_{ud\widetilde{w}}\cdot a^{\widetilde{{\sigma}}} \cdot b^{\widetilde{\tau }}.$$
Again, probabilistically this corresponds to 
 $$ \mathbb{E} \insquare{\prod_{j=k+1}^m a_j^{X_j} b_j^{Y_j} \cdot \mathbbm{1}_{X_k=c}\cdot \mathbbm{1}_{Y_k=d} \cdot \mathbbm{1}_{EQ_{k-1}}},$$
 and the lemma follows.
\end{proof}

\begin{lemma}[\cite{AO17}]\label{lem:one_dim}
Suppose $h(x,y) = h_{00}+h_{10}x+h_{01}y+h_{11}xy$ is a bivariate multi-linear polynomial such that $h_{ij}\geq 0$ for all $i,j\in \{0,1\}$ and $h_{10}\cdot h_{01}\leq  h_{00} \cdot h_{11}$, then for every $\beta\in \R_{\geq 0}$
$$\inf_{x,y>0} \frac{h(x,y)}{x^{\alpha} y^{\alpha}} \alpha^\alpha (1-\alpha)^{1-\alpha}\leq h_{00}+h_{11}.$$
\end{lemma}
\begin{proof}
Fix any $\alpha \geq 0$. It is not hard to prove that for $\alpha>1$, the left hand side of the inequality is actually $0$, hence we can focus on $\alpha \in [0,1]$. Note also that we can assume that $h_{10}\cdot h_{10} = h_{00} \cdot h_{11}$ since if $h_{10} \cdot h_{10} < h_{00} \cdot h_{11}$, we can keep increasing $h_{10}$ until the inequality becomes an equality, this way we might only increase the value of 
$$\inf_{x,y>0} \frac{h(x,y)}{x^{\alpha} y^{\alpha}}$$
but $h_{00} + h_{11}$ stays the same. From $h_{10}\cdot h_{10} = h_{00} \cdot h_{11}$ it then follows that that
$$h(x) = (a_0 + a_1x)(b_0+b_1x)$$
for some $a_0, a_1, b_0, b_1 \geq 0$. Now using Lemma~\ref{lemma:entro_poly} we obtain
\begin{align*}
\inf_{x>0} \frac{a_0+a_1x}{x^\alpha} &= \exp(\kl(a,\wt{\alpha})),\\
\inf_{y>0} \frac{b_0+b_1y}{y^\alpha} &= \exp(\kl(b,\wt{\alpha})).
\end{align*}
Where $a=(a_0, a_1)$, $b=(b_0,b_1)$ and $\wt{\alpha}=(\alpha, 1-\alpha)$.
Therefore
$$\inf_{x,y>0} \frac{h(x,y)}{x^{\alpha} y^{\alpha}} \alpha^\alpha (1-\alpha)^{1-\alpha} = \exp(\kl(ab,\wt{\alpha})).$$
Where $ab=(a_0b_0, a_1b_1)$. What then remains to prove is that 
$$\kl(ab,\wt{\alpha}) \leq \log (a_0 b_0+a_1b_1).$$
However, this follows from the fact that the KL-divergence between two probability distributions $p,q\in \Delta_2$ is nonnegative, when applied to: $(p_1, p_2) = (\alpha, 1-\alpha)$ and $(q_1, q_2) = \inparen{\frac{a_0b_0}{a_0b_0+a_1b_1}, \frac{a_1b_1}{a_0b_0+a_1b_1}}$.
\end{proof}

\begin{lemma}\label{lemma:technical}
Let $q,r$ be distributions over $\{0,1\}^m$ satisfying the IPC property. Then
$$ \sup_{\beta \in [0,1]^m} \insquare{  \beta^\beta (1-\beta)^{1-\beta} \inf_{y,z>0}\frac{q(z)}{z^\beta} \cdot \frac{r(y)}{y^\beta}} \leq \sum_{{\sigma}} q_{\sigma} r_{\sigma}.$$
\end{lemma}
\begin{proof}
We proceed by induction.  Observe first that $ \sum_{{\sigma}\in \{0,1\}^m} q_{\sigma} r_{\sigma}$ can be obtained from $q(z) \cdot r(y)$ by applying a sequence of differential operators. More precisely, let 
\begin{align*}
g_0(z_1, \ldots, z_m,y_1, \ldots, y_m) &:=q(z_1, \ldots, z_m) r(y_1, \ldots, y_m),\\
g_k(z_{k+1}, \ldots, z_m,y_{k+1}, \ldots, y_m) & :=\Phi_{z_k,y_k} \inparen{g_{k-1}(z_k, \ldots, z_m,y_k, \ldots, y_m) }.
\end{align*}

\noindent 
Note that $g_m$ is a constant polynomial given by \begin{equation}\label{eq:deriv}
g_m= \sum_{{\sigma}\in \{0,1\}^m} q_{\sigma} r_{\sigma}.
\end{equation}
Let us fix $\beta \in [0,1]^m$, we prove that for every $k=0,1, \ldots, m$
\begin{equation}\label{eq:ind_hyp}
\prod_{j=k+1}^m \beta_j^{\beta_j} (1-\beta_j)^{1-\beta_j} \inf_{\wt{z},\wt{y}>0} \frac{g_k(\wt{z},\wt{y})}{\prod_{j=k+1}^m z_j^{\beta_j} y_j^{\beta_j}}
\leq \sum_{{\sigma}\in \{0,1\}^m} q_{\sigma} r_{\sigma}
\end{equation}
Where $\wt{y}=(y_{k+1}, \ldots, y_m)$ and $\wt{z}=(z_{k+1}, \ldots, z_m)$ 
Note that for $k=0$ we obtain the lemma. We proceed by induction starting from the base case $k=m$ and go backwards with $k=m-1, \ldots, 1, 0$. The base case follows directly (with equality) from~\eqref{eq:deriv}. Suppose now that $k\in \{1, \ldots, m\}$ and~\eqref{eq:ind_hyp} has been proved for all $k'$ with $k'\geq k$, we prove it for $k-1$. 
\end{proof}
Let us fix $\eps>0$ and values $a_{k+1}, \ldots, a_m, b_{k+1}, b_m >0$ such that 
$$\prod_{j=k+1}^m \beta_j^{\beta_j} (1-\beta_j)^{1-\beta_j}   \frac{g_k(a,b)}{\prod_{j=k+1}^m a_j^{\beta_j} b_j^{\beta_j}}
\leq \eps+\sum_{{\sigma}\in \{0,1\}^m} q_{\sigma} r_{\sigma}.$$
It remains to show that
\begin{equation}\label{eq:to_show}
\inf_{z_k,y_k>0} \frac{g_{k-1}(z_k, a, y_k,b)}{z_k^{\beta_k} y_k^{\beta_k}} \leq g_k(a,b).
\end{equation}
Towards this, write $g_{k-1}(z_k, a, y_k, b)$ as a polynomial in $z_k,y_k$
$$
g_{k-1}(z_k, a, y_k, b) = g_{k-1}^{00}(a,b) +g_{k-1}^{10}(a,b)  z_k + g_{k-1}^{10}(a,b) y_k + g_{k-1}^{11}(a,b)  z_k y_k,
$$
and note that~\eqref{eq:to_show} follows from Lemma~\ref{lem:one_dim} if only we can justify its assumption, that is
$$g_{k-1}^{10}(a,b) \cdot g_{k-1}^{01}(a,b) \leq g_{k-1}^{00}(a,b)\cdot g_{k-1}^{11}(a,b).$$
The above follows from the IPC property because from Lemma~\ref{lemma:poly-and-propx} we have
$$E_{k-1}(a,b;c,d) \propto g_{k-1}^{cd}(a,b)$$
for every $c,d\in \{0,1\}$.

\begin{proofof}{of Lemma~\ref{lemma:main}}
From Theorem~\ref{theorem:poly} the Bethe approximation can be stated in the form of a polynomial optimization problem
$$Z_B(G) = \max_{\beta } \insquare{ \prod_{e\in E} \beta_e^{\beta_e} (1-\beta_e)^{1-\beta_e} \inf_{u>0} \prod_{a\in F} \frac{h_a(x_a)}{x_a^{\beta_a}} }.$$
Let $L\cup R=F$ be the bipartition of the set of factor nodes, i.e., there is no edge $e\in E$ within $L$ or $R$. In other words, every edge $e$ has one endpoint $a^L_{e}\in L$ and one endpoint $a^R_{e}\in R$, or in other words $e= \{a^{L}_e, a^R_e\}$. 

Let us split the product $\prod_{a\in F} \frac{h_a(x_a)}{x_a^{\beta_a}}$ into two parts
\begin{equation}\label{eq:prod}
\prod_{a\in F} \frac{h_a(x_a)}{x_a^{\beta_a}}=\prod_{a\in L}\frac{h_a(x_a)}{x_a^{\beta_a}} \cdot \prod_{a\in R}\frac{h_a(x_a)}{x_a^{\beta_a}},
\end{equation}
corresponding to the bipartition. Let us now rename the variables in the above. For an edge $e\in E$ and $a=a_e^L$ we rename the variable $x_{a,e}$ to $z_e$. Similarly, if $a=a_e^R$ we rename $x_{a,e}$ to $y_e$. Because the factor graph $G$ is bipartite, the product~\eqref{eq:prod} can be then rewritten as
$$\prod_{a\in L} \frac{h_a(z_a)}{z_a^{\beta_a}}\cdot \prod_{a\in R} \frac{h_a(y_a)}{y_a^{\beta_a}}.$$
In the above $z_a=\{z_e\}_{e\in \nhd{a}}$, similarly for $y_a$. Let us now define two polynomials $q,r$ as follows
\begin{align*}
q(z) &= \prod_{a\in L} q_a(z_a),\\
r(y) &= \prod_{a\in R} q_a(y_a).
\end{align*}
The expression~\eqref{eq:prod} can be then further simplified to
$$\frac{q(z)}{z^\beta} \cdot \frac{r(y)}{y^\beta}.$$
Consequently, we arrive at the following form of the Bethe partition function
$$Z_B(G) = \max_{\beta \in [0,1]^m} \insquare{  \beta^\beta (1-\beta)^{1-\beta} \inf_{y,z>0}\frac{q(z)}{z^\beta} \cdot \frac{r(y)}{y^\beta}}.$$
Since by the assumption, the corresponding distributions $(q,r)$ satisfy the IPC property, Lemma~\ref{lemma:technical} implies that 
$$\max_{\beta} \insquare{  \beta^\beta (1-\beta)^{1-\beta} \inf_{y,z>0}\frac{q(z)}{z^\beta} \cdot \frac{r(y)}{y^\beta}} \leq \sum_{{\sigma}\in \{0,1\}^E} q_{\sigma} r_{\sigma}.$$
It remains to observe that $\sum_{{\sigma}\in \{0,1\}^E} q_{\sigma} r_{\sigma}=Z(G)$. To prove it, let us first interpret what the coefficients $q_{\sigma}, r_{\sigma}$ mean in terms of the underlying factor graph.  It is not hard to see that
$$q_{\sigma} = \prod_{a\in L} g_a({\sigma}_{\nhd{a}}),~~~~~~~~r_{\sigma} =\prod_{a\in R} g_a({\sigma}_{\nhd{a}}).$$
Hence
$$\sum_{{\sigma}\in \{0,1\}^E} q_{\sigma} r_{\sigma}= Z(G),$$
which concludes the proof.
\end{proofof}

\begin{remark}\label{rem:ao_sv}
The relaxation  $$\sup_{\beta \in [0,1]^m} \insquare{  \beta^\beta (1-\beta)^{1-\beta} \inf_{y,z>0}\frac{q(z)}{z^\beta} \cdot \frac{r(y)}{y^\beta}}$$ was recently studied in~\cite{AO17} as a way to approximate the inner product $\inangle{q,r}:= \sum_{\sigma}q_{\sigma}r_{\sigma}$. 
The proof of Lemma~\ref{lemma:technical} borrows from ideas developed in this paper.
Interestingly, it follows from the proof of Lemma~\ref{lemma:main} that the relaxation studied in~\cite{AO17} arises as a Bethe approximation of a certain factor graph. To see this, consider a NFG $G$ with only two factors $Q$ and $R$ and $m$ edges between them. The local functions are defined as $g_Q(\sigma) = q_\sigma$ and $g_R(\sigma)=r_\sigma$.

In \cite{SV17}, a similar relaxation $$\sup_{\beta \in [0,1]^m} \inf_{y,z>0} \frac{q(y)r(z)}{y^\beta z^\beta}$$ is considered\footnote{The problem considered in~\cite{SV17} is in fact slightly different: computing $\sum_{\sigma \in \cB} q_\sigma$ for a given family $\cB \subseteq \{0,1\}^m$. The relaxation under discussion is a simple variant of it for computing $\sum_{\sigma} q_\sigma r_\sigma$.}.  When compared to Bethe approximation, this relaxation does not include the negative entropy term $H(\beta)$. This corresponds to a slightly different heuristics for simplifying~\eqref{eq:entropy_prog}. 
\end{remark}

\subsection{Proof of Lemma~\ref{lemma:rs-x}}\label{ssec:proof_rs}

\begin{proof}
Consider the polynomials $q(z)$ and $p(y)$ as constructed in the proof of Lemma~\ref{lemma:main}. When written in the form
$$q(z)=\sum_{{\sigma}\in \{0,1\}^m} q_{\sigma} z^{\sigma}, ~~~~~~~ r(y)=\sum_{{\sigma}\in \{0,1\}^m} r_{\sigma} y^{\sigma},$$
the coefficients satisfy, for every ${\sigma}\in \{0,1\}^m$
$$q_{\sigma} = \prod_{a\in L} g_a({\sigma}), ~~~~~\mbox{and}~~~~~r_{\sigma} = \prod_{a\in R} g_a({\sigma}).$$
In other words $q_{\sigma} \propto p_{\sigma}^L$ and $r_{\sigma} \propto p_{\sigma}^R$. 

Note that $f(z,y):= q(z) \cdot p(-y)$ is a  real stable polynomial, since $q(z)$ and $p(y)$ are real stable as products of real stable polynomials (see~\cite{Vishnoi-Survey}).

As observed by~\cite{AO17}, if for a multi-affine, real polynomial $h(z,y)$ (for $z=(z_1, \ldots, z_m)$ and $y=(y_1, \ldots, y_m)$), $h(z,-y)$ is real stable then $\widetilde{h}(\tilde{z}, - \widetilde{y})$ is real stable as well, where $\widetilde{y}=(y_2, \ldots, y_m)$, $\widetilde{z} = (z_2, \ldots, z_m)$ and $\widetilde{h}(\widetilde{z},\widetilde{y}) = \Phi_{z_1, y_1}( h)$.

Define a sequence of polynomials $f_0, f_1, \ldots, f_m$ by setting: $f_0 = f(z,y)$ and for $k=1,2, \ldots, m$ 
$$f_k(z_{k+1}, \ldots, z_m, y_{k+1}, \ldots, y_m) := \Phi_{z_k,y_k}( f_{k-1}),$$
By the above stated observation, $f_k(z,-y)$ is real stable, for every $k=1,2, \ldots, m$. 

We will deduce the IPC property from real stability of $f_0, f_1, \ldots, f_m$. Indeed, by Lemma~\ref{lemma:poly-and-propx} we know that for every $k\in \{1,2, \ldots, m\}$, $a,b\in \R^{m-k}_{\geq 0}$ and $c,d\in \{0,1\}$, $E_k(a,b;c,d)$ can be expressed as the appropriate coefficient of the polynomial
$$
h_{k-1}(z_k,y_k) = f_{k-1}(z_k, a,y_k,b)=h_{00} + h_{10} z_k +h_{01}y_k + h_{11}z_ky_k.
$$
Since $f_k(z,-y)$ is real stable, it follows, that $h_{k-1}(z_k,-y_k)$ is real stable. Indeed, this is a consequence of the fact that plugging in real constants into a real stable polynomial preserves real stability (see~\cite{Vishnoi-Survey}). Using a characterization of multilinear real stable polynomials by~\cite{Branden07}, the real stability of $h_{k-1}$ is equivalent to:
$$h_{10} \cdot h_{01} \leq h_{00} \cdot h_{11},$$
hence the IPC property holds.
\end{proof}

\bibliographystyle{abbrv}
\bibliography{references}

\appendix

\section{Example (Permanent)}\label{sec:permanent}
We discuss the problem of computing the permanent of a nonnegative real matrix $A\in \R^{n\times n}$. Recall that
$$\per(A) = \sum_{\sigma \in S_n} \prod_{i=1}^n A_{i,\sigma(i)}.$$
Consider the following factor graph representation of permanents~\cite{WC10,Vontobel13}. The graph $G$ has nodes $a_1, a_2, \ldots, a_n, b_1, b_2, \ldots, b_n$ and there is an edge $(a_i,b_j)$ for every $i,j\in \{1,2, \ldots, n\}$. The local functions are then specified as follows:

$$g_{a_i}({\sigma}_{i,1}, {\sigma}_{i,2}, \ldots, {\sigma}_{i,n})=\begin{cases}
A_{i,j}^{1/2} &\mbox{if } \sum_{k=1}^n {\sigma}_{i,k}=1 
\mbox{ and } {\sigma}_{i,j}=1,\\
0 & \mbox{otherwise.}
\end{cases}
$$
$$g_{b_j}({\sigma}_{1,j}, {\sigma}_{2,j}, \ldots, {\sigma}_{n,j})=\begin{cases}
A_{i,j}^{1/2} &\mbox{if } \sum_{k=1}^n {\sigma}_{k,j}=1  
\mbox{ and } {\sigma}_{i,j}=1,\\
0 & \mbox{otherwise.}
\end{cases}
$$

In other words, the local functions are putting constraints on the configurations saying that exactly one element per row is equal to $1$ and similarly for columns: every column contains a single $1$ and $(n-1)$ $0$s.
This implies that if $\sigma \in \{0,1\}^{n \times n}$ corresponds to a perfect matching $M$ in the complete graph $K_{n,n}$ then $$g(\sigma) = \prod_{i,j} A_{i,j}^{\sigma_{i,j}}$$ and in fact $\sum_{\sigma} g(\sigma) = \per(A)$. 
One can show (see~\cite{Vontobel13}) that the Bethe approximation for this factor graphs takes the form
$$Z_B(G) = \sup_{B\in \Omega_n} \exp \inparen{\sum_{i,j} B_{i,j} \log \frac{A_{i,j}}{B_{i,j}} - \sum_{i,j} H(B_{i,j})},$$
where $\Omega_n$ is the set of all {\it doubly-stochastic} $n\times n$ matrices.
The polynomials corresponding to local functions are linear, of the form
$$h_{a_i}(x) = \sum_{j=1}^n A_{i,j}x_{i,j}, ~~~~~~~~h_{b_j}(y) = \sum_{i=1}^n A_{i,j}y_{i,j}.$$
Note that such polynomials are real stable, as they are linear and all their coefficients are nonnegative. Such a polynomial, when evaluated at a point $z\in \C^{n \times n}$ with $\Im(z_{i,j})>0$ for every $i,j=1,2, \ldots, n$ gives a value which also has a positive imaginary part, and hence is not a zero. 
Therefore, we can apply Theorem~\ref{thm:rs} to yield an alternative proof of the lower bound $Z_B(G) \leq Z(G)$ established in~\cite{Gurvits11,GS14}.

\end{document}